\documentclass[conference]{IEEEtran}
\IEEEoverridecommandlockouts
\let\keptmaketitle\maketitle
\usepackage{hyperref}
\let\maketitle\keptmaketitle

\usepackage[utf8]{inputenc}
\usepackage{cite}
\usepackage{amsmath,amssymb,amsfonts}
\usepackage{algorithmic}
\usepackage{graphicx}
\usepackage{textcomp}
\usepackage{xcolor}
\def\BibTeX{{\rm B\kern-.05em{\sc i\kern-.025em b}\kern-.08em
    T\kern-.1667em\lower.7ex\hbox{E}\kern-.125emX}}

\usepackage{booktabs}
\usepackage{exscale}
\usepackage{color}
\usepackage[numbers]{natbib}

\usepackage{pifont}
\usepackage{tabularx}
\usepackage{multirow}
\usepackage{diagbox}
\usepackage{mleftright}
\usepackage{nicefrac}
\usepackage{xspace}
\usepackage{subcaption}

\usepackage{fixmath}
\usepackage{units}
\usepackage{tabularx}
\usepackage{graphicx}
\usepackage{tikz}
\usetikzlibrary{shapes}

\usepackage{paralist}
\usepackage{enumitem}

\usepackage{amsfonts}
\usepackage{amsmath}
\usepackage{amssymb}
\usepackage{amsthm}
\usepackage{amsfonts}
\usepackage{thmtools}    
\usepackage{mleftright}
\usepackage{thm-restate}
\usepackage[mathic=true]{mathtools}
\graphicspath{{./figures/}}

\newcommand{\cNP}[0]{\ensuremath{\mathsf{NP}}\xspace}
\newcommand{\cAPX}[0]{\ensuremath{\mathsf{APX}}\xspace}

\newcommand{\iso}[0]{\psi}

\newcommand*{\thour}{\text{h\,} }
\newcommand*{\tmin}{^{\prime}\mkern-1.2mu}
\newcommand*{\tsec}{^{\prime\prime}\mkern-1.2mu}

\newcommand{\univ}{\ensuremath{\mathcal{X}}\xspace}

\newcommand{\relmiddle}[1]{\mathrel{}\middle#1\mathrel{}}
\DeclarePairedDelimiter\set{\lbrace}{\rbrace}%
\DeclareMathOperator{\depth}{depth}
\DeclarePairedDelimiter\multiset{\lbrace\!\!\lbrace}{\rbrace\!\!\rbrace}%

\newcommand{\norm}[1]{\left\lVert#1\right\rVert}

\newcommand{\lab}[0]{c}

\let\vec\mathbf

\makeatletter
\renewcommand*\env@matrix[1][*\c@MaxMatrixCols c]{%
  \hskip -\arraycolsep
  \let\@ifnextchar\new@ifnextchar
  \array{#1}}
\makeatother

\usepackage[english,linesnumbered,vlined,ruled,nokwfunc]{algorithm2e}
\DontPrintSemicolon
\SetKwComment{note}{$\triangleright$ }{}
\SetFuncSty{textsc}
\SetCommentSty{textit}
\SetDataSty{texttt}
\SetKwInput{Initial}{Initial}
\SetKwInput{Parameter}{Param.}
\SetKwInput{Input}{Input}
\SetKwInput{Data}{Data}
\SetKwInput{Output}{Output}
\ResetInOut{Result} 
\let\oldnl\nl
\newcommand{\nonl}{\renewcommand{\nl}{\let\nl\oldnl}}
\SetKwProg{customProcedure}{Procedure}{}{}
\newcommand{\Procedure}[2]{\BlankLine\nonl\customProcedure{#1}{#2}}
\SetKw{KwAnd}{and} 
\SetKw{KwOr}{or}
\SetKw{KwXor}{xor}
\SetKw{KwNot}{not}

\theoremstyle{plain}
\newtheorem{theorem}{Theorem}

\newtheorem{lemma}[theorem]{Lemma}

\theoremstyle{definition}

\makeatletter
\def\ps@IEEEtitlepagestyle{%
  \def\@oddfoot{\mycopyrightnotice}%
  \def\@evenfoot{}%
}
\def\mycopyrightnotice{%
\fbox{\parbox{.985\textwidth}{%
\footnotesize%
\textcopyright\ 2019 IEEE. Personal use of this material is permitted. 
Permission from IEEE must be obtained for all other uses, in any current 
or future media, including reprinting/republishing this material for 
advertising or promotional purposes, creating new collective works, for 
resale or redistribution to servers or lists, or reuse of any copyrighted
component of this work in other works.%
}}}
\makeatother

\begin{document}

\title{Computing Optimal Assignments in Linear Time for Approximate Graph Matching}

\author{%
\IEEEauthorblockN{%
Nils M.~Kriege\IEEEauthorrefmark{1}, %
Pierre-Louis Giscard\IEEEauthorrefmark{2}, %
Franka Bause\IEEEauthorrefmark{1}, %
Richard C.~Wilson\IEEEauthorrefmark{3}
}
\IEEEauthorblockA{%
\IEEEauthorrefmark{1}%
\textit{Department of Computer Science}, %
\textit{TU Dortmund University}, %
Dortmund, Germany \\%
\{nils.kriege, franka.bause\}@tu-dortmund.de %
}%
\IEEEauthorblockA{%
\IEEEauthorrefmark{2}%
\textit{LMPA Joseph Liouville},
\textit{Université Littoral Côte d'Opale}, %
Calais, France \\
giscard@univ-littoral.fr
}%
\IEEEauthorblockA{%
\IEEEauthorrefmark{3}%
\textit{Department of Computer Science}, %
\textit{University of York}, %
York, United Kingdom \\%
richard.wilson@york.ac.uk
}%
}

\maketitle

\begin{abstract}
Finding an optimal assignment between two sets of objects is a fundamental 
problem arising in many applications, including the matching of `bag-of-words' 
representations in natural language processing and computer vision.
Solving the assignment problem typically requires cubic time and its pairwise
computation is expensive on large datasets. 
In this paper, we develop an algorithm which can find an optimal assignment in 
linear time when the cost function between objects is represented by a tree distance. 
We employ the method to approximate the edit distance between two graphs by
matching their vertices in linear time.
To this end, we propose two tree distances, the first of which reflects 
discrete and structural differences between vertices, and the second of which 
can be used to compare continuous labels. 
We verify the effectiveness and efficiency of our methods using synthetic and 
real-world datasets.
\end{abstract}


\section{Introduction}
Vast amounts of data are now available for machine learning, including text 
documents, images, graphs and many more. Learning from such data typically 
involves computing a similarity or distance function between the data objects.
Since many of these datasets are large, the efficiency of the comparison methods is 
critical. This is particularly challenging since taking the structure adequately 
into account often is a hard problem.
For example, no polynomial-time algorithms are known even for the basic task 
to decide whether two graphs have the same structure~\cite{Garey1979}.
Therefore, the pragmatic approach of describing such data with a `bag-of-words' 
or `bag-of-features' is commonly used. In this representation, a series of objects 
are identified in the data and each object is described by a label or feature. 
The labels are placed in a bag where the order in which they appear does not matter.

In the most basic form, such bags can be represented by histograms or feature
vectors, two of which are compared by counting the number of co-occurrences of a 
label in the bags. This is a common approach not only for images and
text, but also for graph comparison, where a number of graph kernels have been proposed
which use different substructures as elements of the bag~\cite{Vishwanathan2010,Kriege2019}.
In the more general case, two bags of features are compared by summing over all pairs
of features weighted by a similarity function between the features.
However, in both cases, the method is not ideal, since each feature corresponds to a 
specific element in the data object, and so can correspond to no more than one element in a 
second data object. 
The co-occurrence counting method allows each feature to match to multiple features in 
the other dataset.

A different approach is to explicitly find the best correspondence between the features
of two bags. This can be achieved by solving the (linear) assignment problem 
in $O(n^3)$ time using Hungarian-type algorithms~\citep{Burkard2012}.
When we assume weights to be integers within the range of $[0,N]$, scaling algorithm 
become applicable such as~\cite{Duan2012}, which requires
$O(n^{2.5} \log N)$ time. Several authors studied a geometric version of the problem,
where the objects are points in $\mathbb{R}^d$ and the total distance is to be minimised. 
No subquadratic exact algorithm for this task is known, but efficient approximation 
algorithms exist~\cite{Sharathkumar2012}.
This is also the case for various other problem variants, see \cite{Duan2014} and 
references therein.
These algorithms are typically involved and focus on theoretical guarantees. For practical 
applications often standard algorithms with cubic running time, simple greedy strategies or methods
specifically designed for one task are used.

We briefly summarise the use of assignment methods for comparing structured data
in machine learning with a focus on graphs.
While most kernels for such data are based on co-occurrence counting, there has 
been growing interest in deriving kernels from optimal assignments in recent years.
The \emph{pyramid match kernel} was proposed to approximate correspondences between 
bags of features in $\mathbb{R}^d$ by employing a space-partitioning tree structure 
and counting how often points fall into the same bin~\cite{Grauman2007}.
For graphs, the \emph{optimal assignment kernel} was proposed, which establishes a 
correspondence between the vertices of two graphs using the Hungarian algorithm~\cite{Frohlich2005}.
However, it was soon realised that this approach does not lead to valid kernels 
in the general case~\citep{Vishwanathan2010}. 
Therefore, \citet{Johansson2015} derived kernels from optimal assignments by 
first sampling a fixed set of so-called \emph{landmarks} and representing
graphs by their optimal assignment similarities to landmarks.
\citet{Kriege2016b} demonstrated that a specific choice of a weight function (derived 
from a hierarchy) does in fact generate a valid kernel for the optimal 
assignment method and allows computation in linear time. The approach is not designed 
to actually construct the assignment.

The term \emph{graph matching} refers to a diverse set of techniques which
typically establish correspondences between the vertices (and edges) of two 
graphs~\citep{Conte2004,Foggia2014}. This task is closely related to
the classical \cNP-hard \emph{maximum common subgraph} problem, which asks for the largest
graph that is contained as subgraph in two given graphs~\cite{Garey1979}.
Exact algorithms for this problem haven been studied extensively, e.g., for 
applications in cheminformatics~\cite{Kriege2019x}, where small molecular graphs
with about 20 vertices are compared.
The term \emph{network alignment} is commonly used in bioinformatics, where large
networks with thousands of vertices are compared such as protein-protein 
interaction networks. Methods applicable to such graphs typically cannot guarantee
that an optimal solution is found and often solve the assignment problem as a 
subroutine, e.g.,~\cite{Singh2008,Zhang2016}.
In the following we will focus on the \emph{graph edit distance}, which is one 
of the most widely accepted approaches to graph matching with applications ranging 
from cheminformatics to computer vision~\citep{Stauffer2017}. 
It is defined as the minimum cost of edit operations required to transform one 
graph into another graph. The concept has been proposed for pattern recognition 
tasks more than 30 years ago~\citep{Sanfeliu1983}.
However, its computation is $\mathsf{NP}$-hard, since it generalises the maximum 
common subgraph problem~\citep{Bunke1997}.
The graph edit distance is closely related to the notoriously hard quadratic 
assignment problem~\cite{Bougleux2017}.
Recently several elaborated exact algorithms for computing the graph edit distance
have been proposed~\cite{Gouda2016,Lerouge2017,Chen2019}.
Binary linear programming formulations in combination with highly-optimised 
general purpose solvers are among the most efficient approaches, but are still 
limited to small graphs~\citep{Lerouge2017}.
Even a restricted special case of the graph edit distance is \cAPX-hard~\citep{Lin1994},
i.e., there is a constant $c>1$, such that no polynomial-time algorithm
can approximate it within the factor $c$, unless \mbox{$\mathsf{P}$=\cNP}.
However, heuristics based on the assignment problem turned 
out to be effective tools~\citep{Riesen2009a} and are widely used in 
practice~\citep{Stauffer2017}.
The original approach requires cubic running time, which is still not feasible for large
graphs. Therefore, it has been proposed to use non-exact algorithms for solving 
the assignment problem.
Simple greedy algorithms reduce the running time to 
$O(n^2)$~\cite{Riesen2015a,Riesen2015}.
For large graphs the quadratic running time is still problematic in 
practice.
Moreover, many applications require a large number of distance computations, 
e.g., for solving classification tasks or performing similarity search in graph 
databases~\citep{Zeng2009}.

\subsubsection*{Our contribution}
We develop a practical algorithm, which solves the assignment problem exactly in linear 
time when the cost function is a tree metric. 
A tree metric can be represented compactly by a weighted tree of linear size, which we 
use for constructing an optimal assignment.
We show how to embed sets of objects in an $\ell_1$ space preserving the optimal 
assignment costs.
In order to demonstrate that our approach is---despite its simplicity---suitable 
for challenging real-world problems, we use it for approximating the graph edit distance.
To this end, we propose two techniques for generating trees representing cost functions:
\begin{inparaenum}[(i)]
  \item based on Weisfeiler-Lehman refinement to quantify discrete and structural 
        differences,
  \item based on hierarchical clustering for continuous vertex attributes.
\end{inparaenum}
We show experimentally that our linear time assignment algorithm scales to large graphs 
and datasets. Our approach outperforms both exact and approximate methods for 
computing the graph edit distance in terms of running time and provides 
state-of-the-art classification accuracy.
For some datasets with discrete labels our method even beats these approaches 
in terms of accuracy.

\section{Fundamentals}
We summarise basic concepts and results on graphs, tree distances, the assignment problem
and the graph edit distance.

\subsection{Graph theory}
An undirected \emph{graph} $G=(V,E)$ consists of a finite set $V(G) = V$ of \emph{vertices}
and a finite set $E(G) = E$ of \emph{edges}, where each edge connects two distinct vertices.
We denote an edge connecting a vertex $u$ and a vertex $v$ by $uv$ or $vu$, where both refers
to the same edge.
Two vertices $u$ and $v$ are said to be \emph{adjacent} if $uv \in E$ and
referred to as \emph{endpoints} of the edge $uv$. 
The vertices adjacent to a vertex $v$ in $G$ are denoted by 
$N_G(v) = \{u \in V(G) \mid uv \in E(G)\}$ and referred to as \emph{neighbours} of $v$.
A \emph{path} of length $n$ is a sequence of vertices $(v_0, \dots, v_n)$ such 
that $v_{i}v_{i+1} \in E$ for $0 \leq i < n$.
A \emph{weighted} graph is a graph $G$ endowed with a weight function $w : E(G) \to \mathbb{R}$.
The length of a path in a weighted graph refers to the sum of weights of the edges
contained in the path.
A graph $G'=(V',E')$ is a \emph{subgraph} of a graph $G=(V,E)$, written 
$G' \subseteq G$, if $V' \subseteq V$ and $E' \subseteq E$.
Let $V' \subseteq V$, then $G'=(V',E')$ with $E' = \{ uv \in E \mid u,v \in V' \}$
is said to be the subgraph \emph{induced} by $V'$ in $G$ and is denoted by $G[V']$.
An \emph{isomorphism} between two graphs $G$ and $H$ is a bijection 
$\iso : V(G) \to V(H)$ such that
$uv \in E(G) \Leftrightarrow \iso(u)\iso(v) \in E(H)$ for all $u,v \in V(G)$.
Two graphs $G$ and $H$ are said to be \emph{isomorphic} if
an isomorphism between $G$ and $H$ exists.
An \emph{automorphism} of a graph $G$ is an isomorphism $\iso : V(G) \to V(G)$.

\subsection{Tree metrics and ultrametrics}
A dissimilarity function $d : \univ \times \univ$ is a metric on \univ, if it is
\begin{inparaenum}[(i)]
 \item non-negative, 
 \item symmetric, 
 \item zero iff two objects are equal and
 \item satisfies the triangle inequality. 
\end{inparaenum}
A metric $d$ on $\univ$ is an 
\emph{ultrametric} if it satisfies the strong triangle inequality 
$d(x,y)\leq \max\{d(x,z),d(z,y)\}$ for all $x,y,z\in \univ$.
A metric $d$ on $\univ$ is a \emph{tree metric} if 
$d(x,y)+d(v,w)\leq \max\{d(x,v)+d(y,w), d(x,w)+d(y,v)\}$ 
for all $v,w,x,y \in \univ$.

These restricted classes of distances can equivalently be defined in terms of
path lengths in weighted trees, which has been investigated in detail, e.g., in
phylogentics~\citep{Semple2003}.
A weighted tree $T$ with positive real-valued edge weights 
$w : E(T) \to \mathbb{R}_{>0}$ represents the distance 
$d : V(T) \times V(T) \to \mathbb{R}_{\geq0}$
defined as 
$d_{(T,w)}(u,v) = \sum_{e \in P(u,v)} w(e)$, 
where $P(u,v)$ denotes the unique path from $u$ to $v$, for all $u, v \in V(T)$.
For every ultrametric $d$ on $\univ$ there is a rooted tree $T$ with leaves $\univ$ 
and positive real-valued edge weights, such that 
\begin{inparaenum}[(i)]
 \item $d$ is the path length between leaves in $T$,
 \item all paths from any leaf to the root have equal length.
\end{inparaenum}
For every tree metric $d$ on $\univ$ there is a tree $T$ with $\univ \subseteq V(T)$
and positive real-valued edge weights, such that $d$ corresponds to the path 
lengths in $T$.
Note that an ultrametric always is a tree metric.
For the clarity of notation we distinguish between the elements of \univ and the 
nodes of a tree $T$ by introducing an injective map $\varphi : \univ \to V(T)$.
We will refer to both a label $u\in\univ$ and the associated node by the same 
letter, with the meaning clear from the context.
We consider the distance $d :\univ \times \univ \to \mathbb{R}_{\geq 0}$
defined as $d(x,y) = d_{(T,w)}(\varphi(x),\varphi(y))$.

\subsection{The assignment problem}
\label{assignmentproblem}
The assignment problem is a well-studied classical combinatorial problem~\citep{Burkard2012}.
Given a triple $(A, B, c)$, where $A$ and $B$ are sets of distinct objects 
with $|A| = |B| = n$ and $c : A \times B \to \mathbb{R}$ a cost function, the problem 
asks for a one-to-one correspondence $M$ between $A$ and $B$ with minimum costs. The cost 
of $M$ is $c(M) = \sum_{(a,b) \in M} c(a,b)$.
Assuming an arbitrary, but fixed ordering of the elements of $A$ and $B$, an assignment
instance can also be given by a cost matrix $\vec{C} \in \mathbb{R}^{n \times n}$, where
$c_{i,j} = c(a_i, b_j)$ and $a_i$ is the $i$th element of $A$ and $b_j$ is the $j$th 
element of $B$. Note that in this case the input $\vec{C}$ is of size $\Theta(n^2)$, 
where otherwise the input size depends on the representation of the cost function $c$.
The assignment problem is equivalent to finding a minimum weight perfect matching in a 
complete bipartite graph on the two sets $A$ and $B$ with edge weights according to $c$.
Unless $\vec{C}$ is sparse or contains only integral values from a bounded interval,
the best known algorithms require cubic time, which is achieved by the well-known
Hungarian method.

Here we consider the assignment problem for two sets of objects where
the objects are labelled by elements of $\univ$. There exists a labelling
$\ell : A\cup B \to \univ$ associating each object $a\in A\cup B$ with a label
$\ell(a)\in\univ$. Furthermore, we may associate objects with tree nodes using
the map $\varrho=\varphi\cdot\ell$.
We then have $c(a,b)=d(\ell(a),\ell(b))=d_{(T,w)}(\varrho(a),\varrho(b))$.  
We denote by $D^c_\mathrm{OA}(A,B)$ the cost of an optimal assignment between 
$A$ and $B$, and the assignment problem as the quadruple $(A,B,T,\varrho)$.

\subsection{The graph edit distance}\label{sec:ged}
The graph edit distance measures the minimum cost required to transform a graph 
into another graph by adding, deleting and substituting vertices and edges.
Each edit operation $o$ is assigned a cost $c(o)$, which may depend on the attributes
associated with the affected vertices and edges. A sequence of $k$ edit operations
$(o_1,\dots,o_k)$ that transforms a graph $G$ into another graph $H$ is called
an \emph{edit path} from $G$ to $H$. We denote the set of all possible edit
paths from $G$ to $H$ by $\Upsilon(G,H)$.
Let $G$ and $H$ be attributed graphs, the \emph{graph edit distance} from $G$ to 
$H$ is defined by
\begin{equation*}
 d(G,H) = \min \set*{ \sum_{i=1}^k\, c(o_i) \relmiddle|  (o_1,\dots,o_k) \in \Upsilon(G,H) } \,. 
\end{equation*}
In order to obtain a meaningful measure of dissimilarity for graphs, a cost
function must be tailored to the particular attributes that are present in the
considered graphs. 

\subsubsection*{Approximating the graph edit distance by assignments}
Computing the graph edit distance is an \cNP-hard problem and solving practical 
instances by exact approaches is often not feasible.
Therefore, \citet{Riesen2009a} proposed to derive a suboptimal edit path between
graphs from an optimal assignment of their vertices, where the assignment costs
also encode the local edge structure.
For two graphs $G$ and $H$ with vertices $U = \{u_1,\dots,u_n\}$ and 
$V = \{v_1,\dots,v_m\}$, an assignment cost matrix $\vec{C}$ is created according 
to
\begin{equation}\label{eq:cost_matrix}
\vec{C}=
 \begin{bmatrix}[cccc|cccc]
c_{1,1} & c_{1,2} & \cdots & c_{1,m} & c_{1,\epsilon} & \infty & \cdots  & \infty \\
c_{2,1} & c_{2,2} & \cdots & c_{2,m} & \infty & c_{2,\epsilon} & \ddots & \vdots \\
\vdots & \vdots & \ddots & \vdots & \vdots & \ddots & \ddots &\infty \\
c_{n,1} & c_{n,2} & \cdots & c_{n,m} & \infty & \cdots & \infty  & c_{n,\epsilon} \\

\addlinespace[-0.2\aboverulesep]\cmidrule[0.4pt](l{4pt}r{4pt}){1-8}\addlinespace[-1.0\belowrulesep]

c_{\epsilon,1} & \infty & \cdots & \infty & 0 & 0 & \cdots & 0 \\
\infty & c_{\epsilon,2} & \ddots & \vdots & 0 & 0 & \ddots & \vdots \\
\vdots & \ddots & \ddots & \infty & \vdots & \ddots & \ddots & 0 \\
\infty & \cdots & \infty & c_{\epsilon,m} & 0 & \cdots & 0  & 0 \\
\end{bmatrix},
\end{equation}
where the entries are estimations of the cost for substituting, deleting and 
inserting vertices in $G$.
In more detail, the entry $c_{i,\epsilon}$ is the cost for deleting $u_i$
increased by the costs for deleting the edges incident to $u_i$. 
The entry $c_{\epsilon,j}$ is the cost for inserting $v_j$ and all edges 
incident to $v_j$. Finally $c_{i,j}$ is the cost made up of
the cost for substituting the vertex $u_i$ by $v_j$ and the cost of an optimal assignment 
between the incident edges w.r.t.\@ the edge substitution, deletion and insertion
costs. 
An optimal assignment for $\vec{C}$ allows to derive an edit path between $G$ and 
$H$. Its cost is not necessarily minimum possible, but \citet{Riesen2009a} show
experimentally that this procedure leads to a sufficiently good approximation of 
the graph edit distance for many real-world problems.

The costs derived by \citet{Riesen2009a} are directly related to edit costs
of various operations on the graph, but unfortunately these costs are not
suitable for our optimal assignment strategy, which must utilise a tree metric.
For this reason, we use a different set of costs, described in Section~\ref{sec:linear_ged}.
The optimal assignment is recovered using the method described in the next section.
This assignment then induces an edit path which is used to compute a good
approximation to the edit distance.

\section{Optimal assignments under a tree metric}
\label{optimal}
We consider the assignment problem under the assumption that the costs are
derived from a tree metric and propose an efficient algorithm for constructing 
a solution. For a dataset of sets of objects we obtain a distance-preserving 
embedding of the pairwise optimal assignment costs into an $\ell_1$ space.

\subsection{Structural results}
\begin{figure*}

\tikzstyle{vertex}=[circle, fill=black, scale=.65]
\tikzstyle{setA}=[diamond, thick, draw=black, scale=.6,fill=orange, node distance=.7cm]
\tikzstyle{setB}=[circle, thick, draw=black, scale=.8,fill=blue, node distance=.7cm]
\tikzstyle{edge} = [draw,thick,-]
\tikzstyle{dots} = [line width=1.2pt, line cap=round, dash pattern=on 0pt off 3\pgflinewidth]

\begin{subfigure}[b]{0.4\textwidth}\centering
	\begin{tikzpicture}[auto,scale=.7]

	\node[vertex,label=left:{$t$}]  (t) at (-1,-.7) {};
	\node[vertex,label=left:{$s$}]  (s) at (-1,.7) {};
	\node[vertex,label=above:{$u$}] (u) at (0,0) {};
	\node[vertex,label=above:{$v$}] (v) at (1.2,0) {};
	\node[vertex,label=right:{$x$}] (x) at (2.2,-.7) {};
	\node[vertex,label=right:{$w$}] (w) at (2.2,.7) {};

	\node[below of = t, node distance=.4cm] {
	\tikz\node[setA] (aA1)  {};
	\tikz\node[setA] (aA2)  {};
	\tikz\node[setB] (aB1)  {};
	};
	\node[above of = s, node distance=.4cm] {
	\tikz\node[setB] (bB1)  {};
	\tikz\node[setB] (bB2)  {};
	};
	\node[below of = x, node distance=.4cm] {
	\tikz\node[setB] (sB1)  {};
	};
	\node[above of = w, node distance=.4cm] {
	\tikz\node[setA] (tA1)  {};
	\tikz\node[setA] (tA2)  {};
	\tikz\node[setA] (tA3)  {};
	\tikz\node[setB] (tB3)  {};
	};

	\draw[edge] (t) -- (u);
	\draw[edge] (s) -- (u);
	\draw[edge] (u) -- (v);
	\draw[edge] (v) -- (w);
	\draw[edge] (v) -- (x);
	\end{tikzpicture}

\subcaption{Tree representing a metric}
\label{fig:tree_metric_assignment:a}

\end{subfigure}
\begin{subtable}[b]{0.6\textwidth}

  \begin{tabular}{l|ccccc|ccccc}
    \toprule
    \diagbox[width=5em]{\small{Set}}{\small{Index}} &
    $\overleftarrow{su}$ & 
    $\overleftarrow{tu}$ & 
    $\overleftarrow{uv}$ & 
    $\overleftarrow{vw}$ & 
    $\overleftarrow{vx}$ &
    $\overrightarrow{su}$ & 
    $\overrightarrow{tu}$ & 
    $\overrightarrow{uv}$ & 
    $\overrightarrow{vw}$ & 
    $\overrightarrow{vx}$ \\
    \midrule
    $A$ \tikz\node[setA](){}; & 0 & 2 & 2 & 2 & 5 & 5 & 3 & 3 & 3 & 0 \\
    $B$ \tikz\node[setB](){}; & 2 & 1 & 3 & 4 & 4 & 3 & 4 & 2 & 1 & 1 \\
    \bottomrule
  \end{tabular}

\subcaption{Partition size w.r.t.\@ oriented edges}
\label{fig:tree_metric_assignment:b}
\end{subtable}
\caption[]{
A tree representing a metric and the objects of an assignment instance
associated to its nodes. In this instance $A=\{a_1,a_2,a_3,a_4,a_5\}$ and
these objects are labelled $(t,t,w,w,w)$. Similarly the five objects in $B$ are
labelled $(s,s,t,w,x)$. The objects are associated to tree nodes by the map
$\varrho$, and \tikz\node[setA](){}; 
denotes the elements of the set $A$ and \tikz\node[setB](){}; of $B$, respectively.}
\label{fig:tree_metric_assignment}
\end{figure*}
Let $(A, B, T,\varrho)$ be an assignment instance  as described in 
Section~\ref{assignmentproblem}.
We associate the objects of $A$ and $B$ with the nodes of the tree $T$ according 
to the map $\varrho$, cf.~Figure~\ref{fig:tree_metric_assignment:a}.
An assignment $M$ between the objects of $A$ and $B$ is associated with a
collection of paths $\mathcal{P}$ in $T$, such that there is a bijection
between pairs $(a,b) \in M$ and paths 
$(\varrho(a), \dots, \varrho(b)) \in \mathcal{P}$.
In particular, the cost of the assignment equals the sum of path lengths, i.e.,
\begin{equation}\label{eq:path_collection}
c(M) = \sum_{P \in \mathcal{P}}\sum_{e \in P} w(e).
\end{equation}
We do not construct the set $\mathcal{P}$ explicitly, but use this notion to 
develop efficient methods and prove their correctness.
Using Eq.~\eqref{eq:path_collection}, we can attribute the total costs of
an optimal assignment to the individual edges by counting how often they occur 
on paths.
Deleting an arbitrary edge $uv$ yields two connected components, one containing the node
$u$ and the other containing $v$. Let $A_{\overleftarrow{uv}}$ and 
$A_{\overrightarrow{uv}}$ denote the number of objects in $A$ associated by 
$\varrho$ with nodes in the connected component containing $u$ and $v$, 
respectively, cf.~Figure~\ref{fig:tree_metric_assignment:b}.

\begin{lemma}\label{lem:edges}
Let $\mathcal{P}$ be the collection of paths associated with an optimal 
assignment between $A$ and $B$ under a cost function represented by the weighted
tree $T$.
Each edge $uv$ in $T$ appears 
$|A_{\overleftarrow{uv}} - B_{\overleftarrow{uv}}|=|A_{\overrightarrow{uv}} - B_{\overrightarrow{uv}}|$
times in a path in $\mathcal{P}$.
\end{lemma}
\begin{proof}
Splitting $T$ at an edge defines a bipartition of $A$ and $B$.
Since $|A| = |B|$ holds, it follows that 
$|A_{\overleftarrow{uv}} - B_{\overleftarrow{uv}}|=|A_{\overrightarrow{uv}} - B_{\overrightarrow{uv}}|$
for every edge $uv$. 
When $uv$ appears in $|A_{\overleftarrow{uv}} - B_{\overleftarrow{uv}}|$ assignment
paths, the maximum number of assignments is made within each subset, with all
of the smaller of $A$ and $B$ assigned within the subset.
We may assign at most $\min\{ A_{\overleftarrow{uv}}, B_{\overleftarrow{uv}} \}$ 
objects within the connected component containing $u$, and at least the remaining 
$|A_{\overleftarrow{uv}} - B_{\overleftarrow{uv}}|$ objects must be assigned to 
objects in the connected component containing $v$. Therefore, $uv$ appears at 
least this number of times in paths in $\mathcal{P}$.

It remains to be shown that the assignment cannot be optimal when the edge $uv$ 
is contained in more paths.
Assume $\mathcal{P}'$ corresponds to an optimal solution and contains the edge 
$uv$ more than $|A_{\overleftarrow{uv}} - B_{\overleftarrow{uv}}|$ times. 
Then, there are $a_1 \in A$ and $b_2 \in B$ in the connected component 
containing $u$, which are both assigned across the partition to elements 
$b_1 \in B$ and $a_2 \in A$, respectively, in the component containing $v$.
The corresponding assignment paths $P_1 = (a_1,\dots, u,v, \dots, b_1)$ and 
$P_2 = (b_2,\dots, u,v, \dots, a_2)$ are contained in $\mathcal{P}'$.
Consider the paths $P'_1 = (a_1,\dots, u, \dots, b_2)$ and 
$P'_2 = (a_2,\dots, v, \dots, b_1)$, which both do not contain $uv$.
The collection of paths $\mathcal{P}'' =\mathcal{P}' \cup \{P'_1, P'_2\} \setminus \{P_1, P_2\}$
also defines an assignment, where the edges are contained in the same number of 
paths with exception of the edge $uv$, which appears in two paths less.
Since $w(uv)>0$, the associated solution has cost $c(\mathcal{P}'') = c(\mathcal{P}') - 2w(uv)$
and hence $\mathcal{P}'$ cannot correspond to an optimal solution, contradicting
the assumption. 
\end{proof}

This result allows us to compute the optimal assignment cost as a weighted sum
over the edges in the tree representing the cost metric.

\begin{theorem}\label{thm:oa_dist}
Let $(A,B,T,\varrho)$ be an assignment instance with tree edge weights $w$. The cost of an optimal assignment is
$$D^c_\mathrm{OA}(A,B)  = \sum_{uv \in E(T)} |A_{\overleftarrow{uv}} - B_{\overleftarrow{uv}}| \cdot w(uv).$$
\end{theorem}
\begin{proof}
 Directly follows from Eq.~\eqref{eq:path_collection} and Lemma~\ref{lem:edges}.
\end{proof}

\subsection{Constructing an optimal assignment}

In order to compute an optimal assignment, and not just its cost, we again 
associate the objects of $A$ and $B$ with the nodes of the tree $T$.
Then we pick an arbitrary leaf $v$ and match the maximum possible number of 
elements between the subsets of $A$ and $B$ associated with $v$.
The remaining objects are passed to its neighbour and the  considered leaf is 
deleted. Iterating the approach until the tree is eventually empty, yields an 
assignment between all objects of $A$ and $B$.
Algorithm~\ref{alg:compute_assignment} implements this approach. 

\SetKwFunction{PICK}{PickElement}
\SetKwFunction{NEIGHBOUR}{Neighbour}
\SetKwFunction{PAIREL}{PairElements}
\begin{algorithm}[tb]
  \caption{Optimal assignment from a cost tree.}
  \label{alg:compute_assignment}
  \Input{Assignment instance $(A,B,T,\varrho)$.}
  \Data{Subsets $A_v \subseteq A$ and $B_v \subseteq B$ for each vertex $v$ of $T$, partial assignment $M$.}
  \Output{An optimal assignment $M$.}
  \BlankLine

  \lForAll(\note*[f]{Initialisation}){$a \in A$} {
    $A_{\varrho(a)} \gets A_{\varrho(a)} \cup \{a\}$
  }
  \lForAll{$b \in B$} {
    $B_{\varrho(b)} \gets B_{\varrho(b)} \cup \{b\}$
  }
  
  $M \gets \emptyset$ \;
  \While{$V(T) \neq \emptyset$} {
    $v \gets $ arbitrary vertex in $T$ with degree $1$ \label{alg:order}\;
    $M \gets M \cup \PAIREL(A_v,B_v)$ \label{alg:pair_elements:M}\;
    $n \gets \NEIGHBOUR(v)$ \note*[r]{Get distinct neighbour}
    $A_n \gets A_n \cup A_v$ \label{alg:compute_assignment:pass_a}\note*[r]{Pass unmatched objects}
    $B_n \gets B_n \cup B_v$ \label{alg:compute_assignment:pass_b}\;
    $T \gets T \setminus v$ \note*[r]{Remove the node $v$ from $T$}
  }
  \Procedure{\PAIREL{$A_v,B_v$}}{
    $X \gets \emptyset$ \;
    \While{$A_v \neq \emptyset$ \KwAnd $B_v \neq \emptyset$} {
      $a \gets \PICK(A_v)$; \ $A_v \gets A_v\setminus\{a\}$ \label{alg:compute_assignment:poll_a}\;
      $b \gets \PICK(B_v)$; \ $B_v \gets B_v\setminus\{b\}$ \label{alg:compute_assignment:poll_b}\;
      $X \gets X \cup \{(a,b)\}$ \;
    }
    \Return $X$ \;
  }
\end{algorithm}

\begin{theorem}\label{th:assignment_construction}
Algorithm~\ref{alg:compute_assignment} computes an optimal assignment in
time $O(n+t)$, where $n=|A|=|B|$ is the input size and $t=|V(T)|$ the size of 
the tree $T$.
\end{theorem}
\begin{proof}
Since $|A|=|B|$ and every object of $A$ is associated with exactly one object of $B$, the
algorithm constructs an assignment.
The cost of the assignment corresponds to the number of objects that
are passed to neighbours along the weighted edges in lines~\ref{alg:compute_assignment:pass_a}
and \ref{alg:compute_assignment:pass_b}.
Whenever a node $v$ is processed in the while-loop, it has exactly one remaining 
neighbour $n$.
Since $v$ is deleted after the end of the iteration, objects are passed along 
the edge $nv$ only in this iteration. After calling the procedure \PAIREL in
line~\ref{alg:pair_elements:M} either$A_v$ or $B_v$ or both are empty. 
Since all objects in the connected component of $T \setminus nv$ that contains $v$ 
must have been passed to $v$ in previous iterations, exactly 
$|A_{\overleftarrow{nv}} - B_{\overleftarrow{nv}}|$ objects are passed to $n$. 
This is the number of occurrences of $nv$ in every optimal solution according 
to Lemma~\ref{lem:edges}. Therefore, $M$ is an optimal assignment.

The total running time over all iterations for the procedure \PAIREL is $O(n)$, 
the size of the assignment. 
All the other individual operations within the while-loop can be implemented in 
constant time when using linked lists to store and pass the objects.
Therefore the while-loop and the entire algorithm run in $O(n+t)$ total time. 
\end{proof}

Every optimal assignment can be obtained by Algorithm~\ref{alg:compute_assignment} 
depending on the order in which the objects are retrieved by \PICK in 
line~\ref{alg:compute_assignment:poll_a} and~\ref{alg:compute_assignment:poll_b}.

\subsection{Improving the running time}
We consider the setting, where the map $\varrho$ and the weighted tree $(T,w)$ 
encoding the cost metric $c$ are fixed and the distance 
$D^c_\mathrm{OA}$ should be computed for a large number of pairs.
The individual assignment instances possibly only populate a small fraction of 
the nodes of $T$ and only a small subtree may be relevant for 
Algorithm~\ref{alg:compute_assignment}. 
We show that this subtree can be identified efficiently.

Given a tree $T$ and a set $N \subseteq V(T)$, let $T_N$ denote the minimal
subtree of $T$ with $N \subseteq V(T_N)$.
\begin{lemma}\label{lem:relevant_subtree}
  Given a tree $T$ and a set $N \subseteq V(T)$, the subtree $T_N$ can be 
  computed in time $O(|V(T_N)|)$ after a preprocessing step of time $O(|V(T)|)$.
\end{lemma}
\begin{proof}
In the preprocessing step we pick an arbitrary node of $T$ as root and compute 
the depth of every vertex w.r.t.\@ the root using breadth-first search.
Let $d_{\min} = \min \{\depth(v) \mid v \in N\}$. For every node $v \in N$, we
\begin{inparaenum}[(i)]
  \item add $v$ to the result set $R$, and \label{alg:mark:step1}
  \item if the parent $p$ of $v$ is not in $R$ and $\depth(p) \geq d_{\min}$, 
        set $v$ to $p$ and continue with step~\eqref{alg:mark:step1}.
\end{inparaenum}
Let $R_{\min} =\{r \in R \mid \depth(r) = d_{\min}\}$. If $|R_{\min}| > 1$, add 
the parents of all $r \in R_{\min}$ to $R$, decrease $d_{\min}$ by one. Repeat 
this step until $|R_{\min}| = 1$.
Eventually, we have $N \subseteq R$ and $T_N=T[R]$. Every node in $R$ is 
processed only once and the running time is $O(|R|) = O(|V(T_N)|)$.
\end{proof}
Assuming that the tree and the depth of all nodes are given, the result directly
improves the running time of Algorithm~\ref{alg:compute_assignment} to $O(n+|V(T_N)|)$.

\subsection{Embedding optimal assignment costs}
We show how sets of objects can be embedded in a vector space such that the Manhattan 
distance between these vectors equals the optimal assignment costs between sets 
w.r.t.\@ a given cost function.
Let $\mathcal{D} = \{D_1,\dots, D_k\}$ be a dataset with 
$|D_i|=|D_j|$ for all $i, j \in \{1,\dots, k\}$.
Let the cost function $c$ between the objects $\bigcup_i D_i$ be determined by
a tree distance represented by the weighted tree $T$, and the map $\varrho$ from 
the objects to the nodes of the tree.
We consider the following map $\phi_c$ from $A\in\mathcal{D}$ to points 
in $\mathbb{R}^d$ with $d=|E(T)|$ and components indexed by the edges of $T$:
\begin{equation*}
\phi_c(A) = \left[A_{\overleftarrow{uv}} \cdot w(uv)\right]_{uv \in E(T)}. 
\end{equation*}
The Manhattan distance between these vectors is equal to the optimal assignment 
costs between the sets.
\begin{theorem}
Let $c$ be a cost function and $\phi_c$ defined as above, then
$D^c_\mathrm{OA}(A,B) = \norm{\phi_c(A) - \phi_c(B)}_1$.
\end{theorem}
\begin{proof}
We calculate
\begin{align*}
 \norm{\phi_c(A) - \phi_c(B)}_1 &= \hspace{-1em} \sum_{uv \in E(T)} \hspace{-.6em} |A_{\overleftarrow{uv}} \cdot w(uv) - B_{\overleftarrow{uv}} \cdot w(uv)| \\
 &= \hspace{-1em} \sum_{uv \in E(T)} \hspace{-.6em} |A_{\overleftarrow{uv}} - B_{\overleftarrow{uv}}| \cdot w(uv) \\
 &= D^c_\mathrm{OA}(A,B),
\end{align*}%
where the last equality follows from the Theorem~\ref{thm:oa_dist}.
\end{proof}

This makes the optimal assignment costs available to fast indexing methods and 
nearest-neighbour search algorithms, e.g., following the locality sensitive
hashing paradigm.

\section{Approximating the graph edit distance in linear time}\label{sec:linear_ged}
We combine our assignment algorithm with the idea of \citet{Riesen2009a} to 
approximate the graph edit distance detailed in Section~\ref{sec:ged}.
To this end, we propose two methods for constructing a tree distance, such that 
the optimal assignment between the vertices of two graphs w.r.t.\@ to 
these distances is suitable for approximating the graph edit distance.
In order to quantify discrete and structural differences, we propose to use 
the Weisfeiler-Lehman method and, for graphs with continuous labels, hierarchical 
clustering. 
Note that both approaches can be combined to form a tree taking both, discrete 
and continuous labels, into account.
The tree distances we consider are in fact ultrametrics. 
We proceed by a discussion on how to cast the assignment formulation of 
\citet{Riesen2009a} with artificial elements that represent vertex insertion and 
deletion costs to an ultrametric.

\subsection{Ultrametric cost matrices for the graph edit distance}\label{sec:limits}
The cost matrix of Eq.~\eqref{eq:cost_matrix} is easily seen not to 
be in accordance with the strong triangle inequality. Consider the cost
matrix $\vec{C}$ obtained for a graph $G$ with $n$ vertices compared to 
itself.
Let $A=(a_1,\dots, a_{2n})$ and $B=(b_1,\dots, b_{2n})$ and consider 
$c(a_1,b_{n+2})\leq \max\{c(a_1,b_{n+1}),c(b_{n+1},b_{n+2})\}$. We have
$c(a_1,b_{n+2}) = \infty$ and $c(a_1,b_{n+1}) = c_{1,\epsilon}$ and 
$c(b_{n+1},b_{n+2})$ not specified by $\vec{C}$. However, 
$c(b_{n+1},b_{n+2}) \leq \max\{c(a_{n+1},b_{n+1}),c(a_{n+1},b_{n+2})\} = 0$
and, thus, a contradiction to the strong triangle inequality, unless 
$c_{1,\epsilon}=\infty$.
Therefore, we have to modify the definition of the cost matrix. The entries
$\infty$ in the upper right and lower left corner have been introduced with
the argument, that every vertex can be inserted and deleted at most once~\citep{Riesen2009a}.
This, however, is already guaranteed, since the assignment is a bijection.
We simplify the cost matrix as follows
\begin{equation}\label{eq:cost_matrix_ultra}
\vec{C}_\text{ultra}=
 \begin{bmatrix}[cccc|cccc]
c_{1,1} & c_{1,2} & \cdots & c_{1,m} & \tau & \tau & \cdots  & \tau \\
c_{2,1} & c_{2,2} & \cdots & c_{2,m} & \tau & \tau & \ddots & \vdots \\
\vdots & \vdots & \ddots & \vdots & \vdots & \ddots & \ddots &\tau \\
c_{n,1} & c_{n,2} & \cdots & c_{n,m} & \tau & \cdots & \tau  & \tau \\

\addlinespace[-0.2\aboverulesep]\cmidrule[0.4pt](l{4pt}r{4pt}){1-8}\addlinespace[-1.0\belowrulesep]

\tau & \tau & \cdots & \tau & 0 & 0 & \cdots & 0 \\
\tau & \tau & \ddots & \vdots & 0 & 0 & \ddots & \vdots \\
\vdots & \ddots & \ddots & \tau & \vdots & \ddots & \ddots & 0 \\
\tau & \cdots & \tau & \tau & 0 & \cdots & 0  & 0 \\
\end{bmatrix},
\end{equation}
where $\tau$ is the cost for vertex deletion and insertion. Moreover, we 
assume that the vertex substitution costs 
\begin{inparaenum}[(i)]
 \item satisfy $c_{i,j}\leq \tau$ for $i \in \{1,\dots,n\}$, $j \in \{1,\dots,m\}$, and
 \item can be represented by an ultrametric tree.
\end{inparaenum}
We can extend this tree for vertex insertion and deletion as defined by 
$\vec{C}_\text{ultra}$ by adding a node $x$ to the root, where the edge to the parent 
has weight $\tau$.
Just like the matrix~\eqref{eq:cost_matrix_ultra} contains additional rows and 
columns for vertex insertion and deletion, we associate artificial vertices
with the node $x$ via the map $\varrho$. Note that these can be matched at zero 
cost at $x$, which represents the bottom right submatrix of~\eqref{eq:cost_matrix_ultra}
filled with $0$.

\subsection{Weisfeiler-Lehman trees for discrete structural differences}
Weisfeiler-Lehman refinement, also known as colour refinement or na\"ive vertex 
classification, is a classical heuristic for graph isomorphism testing~\citep{Weisfeiler1968,Arvind2015}.
It iteratively refines the discrete labels of the vertices, called colours, 
where in each iteration two vertices with the same colour obtain different new 
colours if their neighbourhood differs w.r.t.\@ the current colouring.
More formally, given a graph $G$ with initial colours $\lab_0$,
a sequence $(\lab_1, \lab_2,\dots)$ of refined colours is computed, where 
$\lab_i$ is obtained from $\lab_{i-1}$ by the following procedure.
For every vertex $v\in V(G)$, sort the multiset of colours 
$\multiset{\lab_{i-1}(u) \mid uv \in E(G)}$ to obtain a unique sequence of colours 
and add $\lab_{i-1}(v)$ as first element. Assign a new colour $\lab_i(v)$ to every 
vertex $v$ by employing an injective mapping from colour sequences to new colours.
Since colours are preserved under isomorphism, a necessary condition for two graphs 
$G$ and $H$ to be isomorphic is 
\begin{equation}\label{eq:wl_iso_test}
 \multiset{\lab_i(v) \mid v \in V(G)} = \multiset{\lab_i(v) \mid v \in V(H)} \quad \forall i \geq 0\,,
\end{equation}
where for both graphs the same injective colour map is used.
Vice versa, the condition \eqref{eq:wl_iso_test} may be satisfied also for two non-isomorphic
graphs $G$ and $H$. 

Each colouring $\lab_i$ induces a partition $\mathcal{C}_i$ of $V(G)$, where the 
vertices with the same colour are in one cell.
The partition $\mathcal{C}_i$ is a refinement of the partition $\mathcal{C}_{i-1}$.
Therefore, colour refinement applied to a set of graphs under the same injective colour 
map yields a hierarchy of partitions of the vertices, which forms a tree. 
We perform colour refinement for a fixed number of $h$ iterations and consider the 
metric induced by the resulting tree. Let $\varrho$ associate the vertices of the graphs 
in the dataset with the node representing their final colour in the tree. 
Then the path length between two nodes in the tree represents the number of refinement 
steps in which the associated vertices have different colours.
Assuming that $h$ is fixed, we obtain a linear running time for approximating
the graph edit distance with Theorem~\ref{th:assignment_construction} and 
Lemma~\ref{lem:relevant_subtree}.

\subsubsection{Relation to Weisfeiler-Lehman graph kernels}

The  Weisfeiler-Lehman method has been used successfully to derive efficient and expressive 
graph kernels~\citep{Shervashidze2011,Kriege2016b}. 
The Weisfeiler-Lehman subtree kernel~\citep{Shervashidze2011} considers each colour as 
a feature and represents a graph by a feature vector, where each component 
counts the number of vertices in the graph having that colour in one iteration.
The Weisfeiler-Lehman optimal assignment kernel~\citep{Kriege2016b} applies the 
histogram intersection kernel to these feature vectors and is equal the optimal 
assignment similarity between the vertices. However, both approaches yield a
similarity measure that crucially depends on the number $h$ of refinement operations.
The similarity keeps changing in value with increasing $h$ even after the 
partition is stable.
Therefore, in classification experiments $h$ is typically determined by cross-validation 
in an computational expensive grid search, e.g., from the set $\{0,\dots,7\}$.
Our approach to approximate the graph edit distance, in contrast, is less sensitive 
to a particular choice of $h$ and can be expected to always benefit from more 
iterations.

\subsubsection{Optimality for amenable graphs}\label{sec:amenable}

We show that a modification of our algorithm actually constructs an isomorphism 
(resulting in an empty edit path with zero costs) between two isomorphic graphs under 
the following assumptions.
First, we assume that $h$ is chosen sufficiently large such that the refinement process 
converges, i.e., the stable partition $\mathcal{C}_h = \mathcal{C}_{h+1}$ is obtained.
Moreover, we assume that the graphs are \emph{amenable} to colour refinement, meaning
that the condition~\eqref{eq:wl_iso_test} is satisfied if and only if $G$ and $H$ are
isomorphic~\cite{Arvind2015}.
If all vertices have distinct colours, which is the case with high probability
for random graphs~\cite{Babai1980}, the graph is amenable and 
our approach constructs an isomorphism without any modification. 
However, this is, for example, not the case for graphs with non-trivial automorphisms, 
which may still be amenable. 
Our graph edit distance approximation proceeds by assigning vertices with the most 
specific colours to each other. For isomorphic graphs, there must be a choice for the 
module \PAIREL such that Algorithm~\ref{alg:compute_assignment} constructs an 
isomorphism, but it does not guarantee that this choice is made.
To achieve this, we have to modify two parts of Algorithm~\ref{alg:compute_assignment}: 
\begin{inparaenum}[(i)]
 \item the function \PAIREL must respect the edges of $G[A_v]$ and $H[B_v]$ when matching
 the vertices as well as the edges to vertices that were already mapped in previous steps,
 \item the ordering in which the leaves are selected in line~\ref{alg:order}
 must guarantee that the partial mapping can be extended to eventually form an isomorphism.
\end{inparaenum}
These steps can be implemented efficiently by inspection of the subgraphs induced by the 
individual colour classes as well as the bipartite graphs containing the edges between two 
colour classes. Due to space limitations we refrain from giving a detailed technical
description and refer the reader to the construction used by \citet{Arvind2015} (Proof of 
Theorem~9).
We denote by $\widehat{\mathit{GED}}(G,H)$ the approximate graph edit distance (assuming 
non-zero edit costs) obtained for $G$ and $H$ from the optimal assignment under the 
Weisfeiler-Lehman tree metric with Algorithm~\ref{alg:compute_assignment} modified as 
described above.

\begin{theorem}
Let $G$ and $H$ be graphs amenable to colour refinement, then 
$\widehat{\mathit{GED}}(G, H) = 0 \Leftrightarrow G$ and $H$ are isomorphic. 
\end{theorem}
\begin{proof}
Since $\widehat{\mathit{GED}}(G, H)$ is the cost of an edit path transforming 
$G$ to $H$, the implication follows directly. 
By the result of \citet{Arvind2015} we obtain that an isomorphism is constructed 
if $G$ and $H$ are isomorphic, which yields an edit path with cost $0$.
\end{proof}

\subsection{Hierarchical clustering for continuous labels}
We apply the bisecting $k$-means algorithm~\citep{Steinbach2000} 
to obtain a hierarchical clustering of the continuous vertex labels of all 
graphs. This is then used as the tree defining the assignment costs. 
We use Lloyd's algorithm~\citep{Lloyd1982} for each $2$-means problem and perform
bisection steps until a fixed number of $l$ leaves is created. 
The $k$-means algorithm is widely popular for its speed in practice, although its
complexity is exponential in the worst-case~\cite{Vattani2011}.
However, \citet{Duda2000} state that the number of iterations until the clustering 
stabilises is often linear or even sublinear on practical data sets.
In any case, hierarchical clustering algorithms with linear worst-case time complexity
are known~\citep{Aggarwal2013}.
Assuming that the clustering is performed in linear time, we also obtain a
linear running time for approximating the graph edit distance as above.

\section{Experimental Evaluation}
Our goal in this section is to answer the following questions experimentally.
\begin{enumerate}[topsep=0pt,itemsep=-1ex,partopsep=1ex,parsep=1ex,leftmargin=*,labelindent=0pt,label=\bfseries\itshape Q\arabic*]
  \item How does our approach scale w.r.t.\@ the graph and dataset size compared to 
        other methods?
  \item How accurately does it approximate the graph edit distance for common datasets?
  \item How does it perform regarding runtime and accuracy in classification tasks?
  \item How does our method compare to other approaches for graph classification?
\end{enumerate}

\subsection{Method}
We have implemented the following methods for computing or approximating the 
graph edit distance in Java using the same code base where possible.
\begin{description}[topsep=0pt,parsep=1ex,itemsep=-1ex]
 \item [Exact] Binary linear programming approach to compute the graph edit
       distance exactly~\citep{Lerouge2017}.
       We implemented the most efficient formulation (F2) and solved all instance 
       using Gurobi 7.5.2.
 \item [BP] Approximate graph edit distance using the Hungarian algorithm to solve 
       the assignment problem as proposed by~\citet{Riesen2009a}.
 \item [Greedy] The greedy graph edit distance proposed by~\citet{Riesen2015a} 
       solving the assignment problem by a row-wise greedy algorithm.
 \item [Linear] Our approach based on assignments under a tree distance.
       For graphs with discrete labels we used Weisfeiler-Lehman trees with $h=7$
       refinement steps, and bisecting $k$-means clustering with $l=300$ 
       leaves otherwise.
\end{description}
The experiments were conducted using Java v1.8.0 on an Intel Core i7-3770 CPU 
at 3.4GHz (Turbo Boost disabled) with 16GB of RAM. The methods BP, Greedy and 
Linear use a single processor only, the Gurobi solver for the Exact method was 
allowed to use all four cores with additional Hyper-Threading.

We used the graph classification benchmark sets contained in the IAM Graph 
Database~\citep{Riesen2008}\footnote{%
Please note that the statistics of the datasets may differ from the datasets 
used in~\citep{Riesen2009a}.} and the repository of benchmark datasets for graph 
kernels~\cite{KKMMN2016}.
The datasets AIDS, Mutagenicity and NCI1 represent small molecules and have discrete 
labels only. The Letter datasets have continuous vertex labels representing
2D coordinates and differ w.r.t.\@ the level of distortion,
(L)--low, (M)--medium and (H)--high. 
The statistics of these graph datasets are summarised in Table~\ref{tab:datasets}.
\begin{table}
	\begin{center}
	\caption{Dataset statistics and properties.}
	\label{tab:datasets}
	\setlength{\tabcolsep}{2pt}
	\begin{tabular}{@{}lrrrrccllc@{}}\toprule
		\multirow{2}{*}{\textbf{Data set}}  &\multicolumn{4}{c}{\textbf{Properties}}&\multicolumn{2}{c}{\textbf{Labels}}&\multicolumn{2}{c}{\textbf{Attributes}}&\multirow{2}{*}{\textbf{Ref.}}\\\cmidrule{2-9}
			&  Graphs& Classes  & $\varnothing |V|$ & $\varnothing |E|$ & Vertex & Edge & Vertex & Edge\\\midrule
	{AIDS} &  2000  & 2 & 15.69 & 16.20 & + & + & -- & -- & \citep{Riesen2008} \\
	{Letter (L)} &  2250  & 15 &  4.68  & 3.13  & -- & -- & + (2) & -- & \citep{Riesen2008} \\
	{Letter (M)} & 2250  & 15 &  4.67  & 4.50  & -- & -- & + (2) & -- & \citep{Riesen2008} \\
	{Letter (H)} & 2250  & 15 &  4.67  & 4.50  & -- & -- & + (2) & -- & \citep{Riesen2008} \\
	{Mutagenicity}  &  4337  & 2 &  30.32  &  30.77  & + & + & -- & -- & \citep{Riesen2008} \\
	{NCI1} &  4110  & 2 & 29.87 & 32.30 & + & -- & -- & -- & \citep{Shervashidze2011} \\
	\bottomrule
	\end{tabular}
	\end{center}
\end{table}
We used the predefined train, test and validation sets when available or 
generated them randomly using $\nicefrac{1}{3}$ of the objects for each set, 
balanced by class label.
We performed $k$-nearest neighbours classification based on the graph edit distance.
The costs for vertex insertion and deletion were both set to $\tau_\text{vertex}$ 
and the costs for insertion and deletion of edges were set to $\tau_\text{edge}$.
The costs for substituting vertices or edges are determined by the Euclidean 
distance in case of continuous labels. In case of discrete labels we assume
cost 0 for equal labels and 1 otherwise.
We use the validation set to select the parameters $k\in \{1,3,5\}$, 
$\tau_\text{vertex} \in\{0.1, 0.5, 0.9, 1.3, 1.7\}$ and 
$\tau_\text{edge} \in\{0.1, 0.5, 0.9, 1.3, 1.7\}$ by grid search.
The approach resembles the experimental settings used by~\citet{Riesen2009a}.
The reported runtimes were obtained using the selected parameters.
In order to systematically investigate the dependence of the runtime on the 
graph size we generated random graphs according to the Erd\H{o}s--R\'enyi model 
with edge probability $0.15$. For these experiments, we set 
$\tau_\text{vertex} = \tau_\text{edge} = 1$.
For the linear method, the runtimes include the time for constructing the tree.

For comparison with other approaches to graph classification, we used two graph 
kernels as a baseline. The \emph{GraphHopper kernel} (GH) \citep{Feragen2013} 
supports graphs with discrete and continuous labels by applying either the Dirac kernel or a
Gaussian kernel. The Weisfeiler-Lehman optimal assignment kernel (WLOA)~\citep{Kriege2016b} 
supports only graphs with discrete labels.
We used the $C$-SVM implementation LIBSVM~\cite{Chang2011}, selecting 
$C \in \{10^{-3},10^{-2},\dots, 10^3\}$ and $h \in \{0,1,\dots,7\}$ using
the validation set.

\subsection{Results}
We report on our experimental results and answer our research questions.

\subsubsection*{\textbf{Q1}}
\begin{figure}[t]
  \begin{subfigure}[c]{0.45\columnwidth}
    \includegraphics[scale=.77]{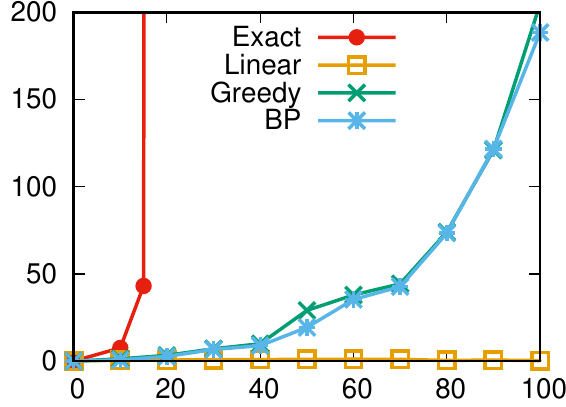}
    \subcaption{Graph size}\label{fig:scaling:graphsize}
  \end{subfigure}\hspace{.4cm}
  \begin{subfigure}[c]{0.45\columnwidth}
    \includegraphics[scale=.77]{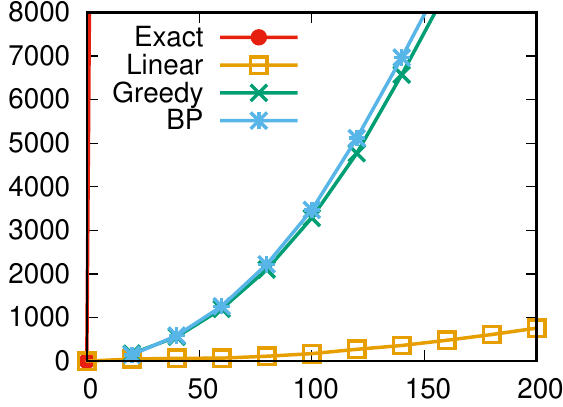}
    \subcaption{Dataset size}\label{fig:scaling:datasetsize}
  \end{subfigure}
  \caption{Runtime in milliseconds to (\subref{fig:scaling:graphsize}) compute 
           the graph edit distance between two random graphs; and 
           (\subref{fig:scaling:datasetsize}) between all pairs of graphs in a 
           dataset of random graphs.}
  \label{fig:scaling}
\end{figure}
Figure~\ref{fig:scaling} shows the growth of the runtime with
increasing graph and dataset size. Our method is the only one of those studied 
that scales to large graphs. The number of distance computations and thus the 
runtime of all methods grows quadratically with the dataset size. Even for the 
small random graphs on 15 vertices we generated, our method is more than one 
order of magnitude faster than other approximate methods.

\subsubsection*{\textbf{Q2}}
\begin{figure*}[t]\centering

  \rotatebox{90}{\hspace{-3.4em}Greedy vs.\@ Linear}\hfill
  \begin{subfigure}{0.19\textwidth}
    \includegraphics[scale=.75]{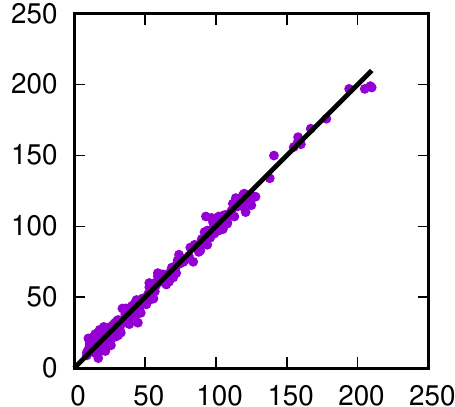}
  \end{subfigure}
  \begin{subfigure}{0.19\textwidth}
    \includegraphics[scale=.75]{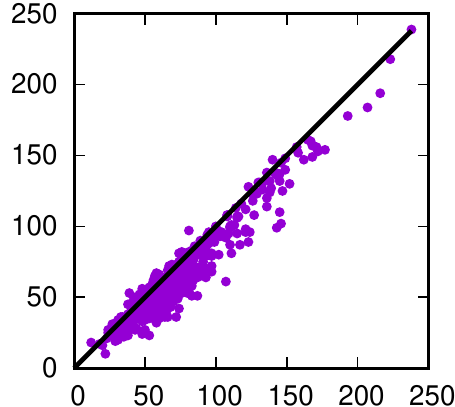}
  \end{subfigure}
  \begin{subfigure}{0.19\textwidth}
    \includegraphics[scale=.75]{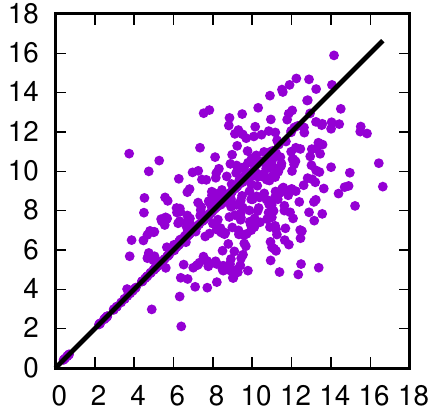}
  \end{subfigure}
  \begin{subfigure}{0.19\textwidth}
    \includegraphics[scale=.75]{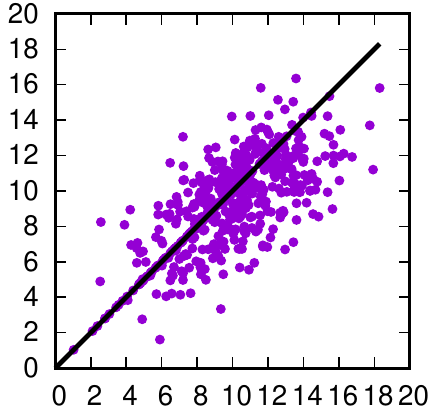}
  \end{subfigure}
  \begin{subfigure}{0.19\textwidth}
    \includegraphics[scale=.73]{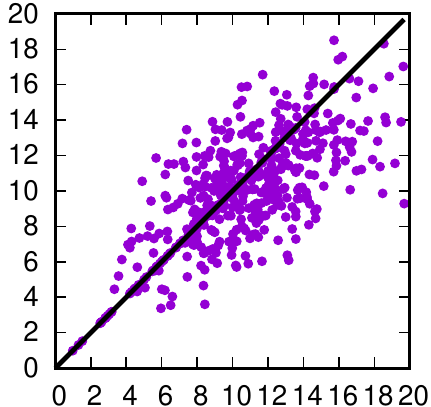}
  \end{subfigure}\vspace{.7em}

  \rotatebox{90}{\hspace{-2.4em}BP vs.\@ Linear}\hfill
  \begin{subfigure}{0.19\textwidth}
    \includegraphics[scale=.75]{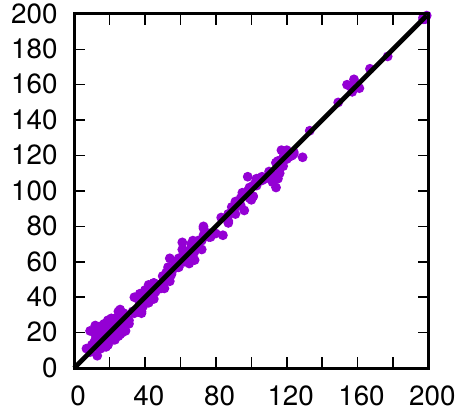}
  \end{subfigure}
  \begin{subfigure}{0.19\textwidth}
    \includegraphics[scale=.75]{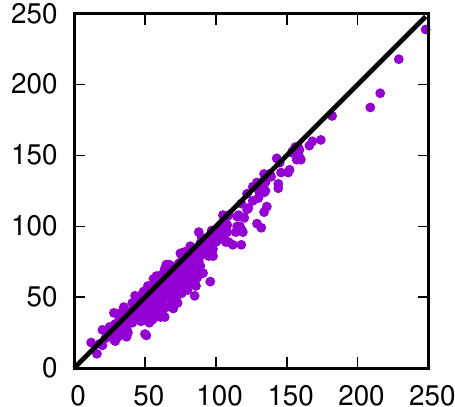}
  \end{subfigure}
  \begin{subfigure}{0.19\textwidth}
    \includegraphics[scale=.75]{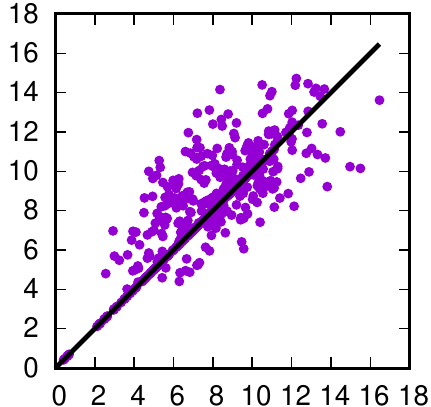}
  \end{subfigure}
  \begin{subfigure}{0.19\textwidth}
    \includegraphics[scale=.75]{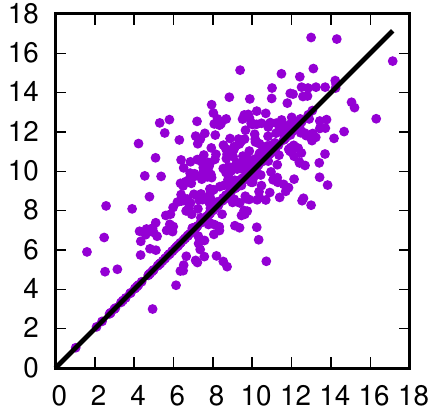}
  \end{subfigure}
  \begin{subfigure}{0.19\textwidth}
    \includegraphics[scale=.75]{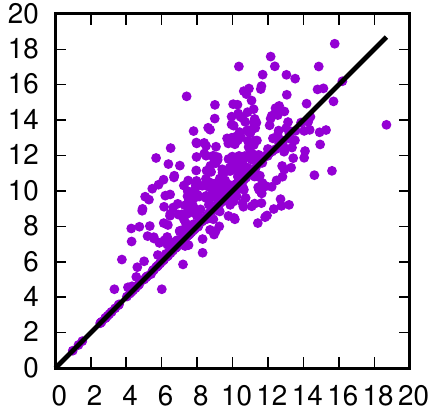}
  \end{subfigure}\vspace{.7em}

  \rotatebox{90}{\hspace{-2em}Exact vs.\@ Linear}\hfill
  \begin{subfigure}{0.19\textwidth}
    \includegraphics[scale=.75]{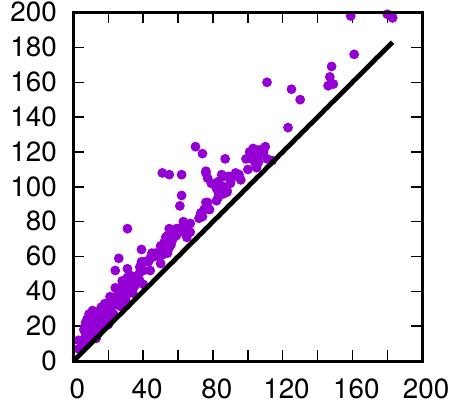}
    \subcaption{AIDS}\label{fig:approx:exact:aids}
  \end{subfigure}
  \begin{subfigure}{0.19\textwidth}
    \includegraphics[scale=.75]{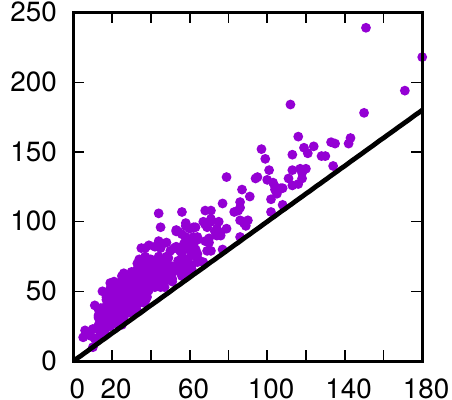}
    \subcaption{Mutagenicity}\label{fig:approx:exact:mutagenicity}
  \end{subfigure}
  \begin{subfigure}{0.19\textwidth}
    \includegraphics[scale=.75]{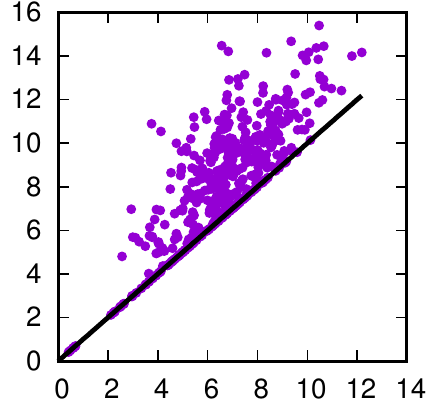}
    \subcaption{Letter (L)}\label{fig:approx:exact:letter_low}
  \end{subfigure}
  \begin{subfigure}{0.19\textwidth}
    \includegraphics[scale=.75]{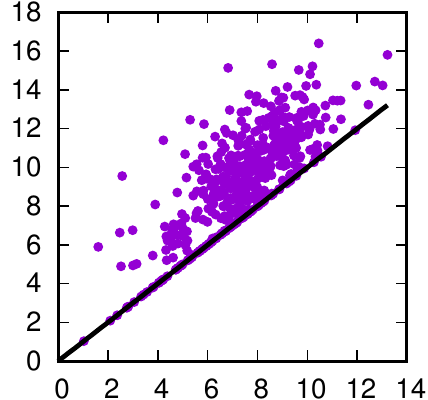}
    \subcaption{Letter (M)}\label{fig:approx:exact:letter_medium}
  \end{subfigure}
  \begin{subfigure}{0.19\textwidth}
    \includegraphics[scale=.75]{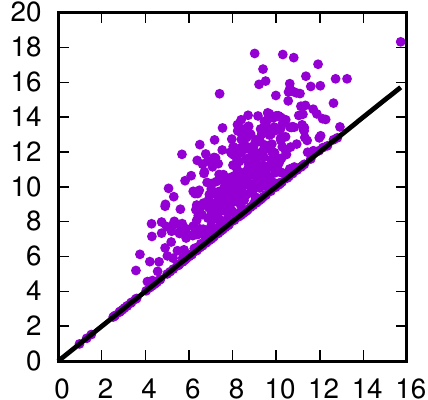}
    \subcaption{Letter (H)}\label{fig:approx:exact:letter_high}
  \end{subfigure}

  \caption{Graph edit distance computed by Linear ($y$-axis) and Exact/Greedy/BP
  ($x$-axis) for 500 randomly chosen pairs of graphs from the IAM graph datasets.}
  \label{fig:approx}
\end{figure*}

\begin{table*}\centering
  \begin{tabular}{lrrrrrr|rrrrrr}
    \toprule
    \multirow{2}{*}{\textbf{Dataset}}    &\multicolumn{6}{c}{\textbf{Accuracy}} &\multicolumn{6}{|c}{\textbf{Runtime}}\\\cmidrule{2-13}
                 & Exact &  BP  & Greedy &  Linear & GH$_\text{SVM}$ & WLOA$_\text{SVM}$ & Exact  &  BP  & Greedy & Linear & GH$_\text{SVM}$ & WLOA$_\text{SVM}$ \\
    \midrule
    AIDS         &---\ \ & 99.6 & 99.6   & 99.6    & 99.5  & 99.6 & ---\hspace{.8em} & $2\tmin13\tsec$ &  $1\tmin35\tsec$ &       $13\tsec$  &$2\thour2\tmin31\tsec$& $<1\tsec$\\
    Mutag.       &---\ \ & 70.7 & 70.8   & 74.4    & 72.3  & 80.8 & ---\hspace{.8em} & $1\thour21\tmin28\tsec$                & $56\tmin37\tsec$ & $3\tmin49\tsec$ &$>24\thour$& $19\tsec$\\
    Letter (L)   &  98.8 & 98.5 & 98.7   & 98.5    & 98.1  &---\ \ &  $8\tmin19\tsec$ &        $9\tsec$ &         $9\tsec$ &        $4\tsec$ & $36\tsec$ & ---\ \ \\
    Letter (M)   &  93.6 & 92.9 & 91.2   & 91.3    & 86.0  &---\ \ & $20\tmin38\tsec$ &       $10\tsec$ &         $9\tsec$ &        $4\tsec$ & $36\tsec$ & ---\ \ \\
    Letter (H)   &  88.4 & 88.1 & 87.2   & 85.2    & 79.2  &---\ \ & $40\tmin21\tsec$ &       $13\tsec$ &        $12\tsec$ &        $4\tsec$ & $52\tsec$ & ---\ \ \\
    NCI1         &---\ \ & 73.5 & 74.7   & 78.1    &---\ \ &81.5& ---\hspace{.8em} & $32\tmin46\tsec$& $26\tmin17\tsec$ &  $2\tmin7\tsec$ & $>24\thour$& $13\tsec$\\
    \bottomrule
  \end{tabular}
\caption{Accuracy and runtime for classifying the elements of the test set, including preprocessing if necessary.}
\label{tab:acc_time}
\end{table*}
To compare how accurately the graph edit distance is computed, we have selected 
500 pairs of graphs at random from each IAM dataset and computed their graph edit 
distance by all four methods.
Figure~\ref{fig:approx} shows how the distance computed by the Linear 
method compares to the distances obtained by the other three methods.
Points below the diagonal line represent pairs of graphs, for which the 
edit distance computed by Linear is actually smaller than the one
computed by the competing approach.
Compared to the Greedy approach the Linear method appears to give slightly better 
results on an average.
For the datasets Mutagenicity, Letter (L), (M) and (H) there are more points 
below the diagonal than above the diagonal.
When comparing to BP, this is still the case for the Mutagenicity dataset, but 
not for the Letter datasets.
This can be explained by the fact that continuous distances for several points 
cannot be represented by a tree metric without distortion.
In order to compare with the exact method on Mutagenicity and AIDS, we 
introduced a timeout of 100 seconds for each distance computation.
This was necessary since hard instances may require 
more than several hours. In case of a timeout the best solution found so far 
is used, which is not guaranteed to be optimal.
The Linear method shows a clear divergence from the solutions of the exact approach, 
in particular for pairs of graphs with a high (optimal) edit distance. 
However, it is likely that non-optimal solutions in this case do not harm a 
nearest neighbours classification.

\subsubsection*{\textbf{Q3}}
Table~\ref{tab:acc_time} summarises the results of the classification 
experiments. The Linear approach provides a high classification accuracy
comparable to BP and Greedy. For the dataset Mutagenicity and NCI1 it even 
performs better than the other approaches. This can be explained by the ability
of the Weisfeiler-Lehman tree to exploit more graph structure than BP.
For the Letter datasets, the Linear method is on a par with the other
methods for the version with low distortion, but performs slightly
worse when the distortion increases. 
This observation is in accordance with the approximation quality achieved for 
the datasets, cf.\@ Figure~\ref{fig:approx}.
The Linear method clearly outperforms all other approaches in terms
of runtime. This becomes in particular clear for the dataset 
Mutagenicity, which contains the largest graphs in the test with 30.32 vertices 
on an average.

\subsubsection*{\textbf{Q4}}
The GraphHopper kernel performs worse than our Linear approach w.r.t.\@
running time and classification accuracy. WLOA can only be applied to 
the molecular datasets with discrete labels. For these it performs 
exceptionally well regarding both accuracy and runtime. 
The result suggests that the notion of similarity provided by the graph 
edit distance is less suitable for this classification task.

\section{Conclusion}
We have shown that optimal assignments can be computed efficiently for tree 
metric costs. Although this is a severe restriction, we designed such
costs functions suitable for the challenging problem of graph matching.
Our approach allows to embed the optimal assignment costs in an $\ell_1$ space. 
It remains future work to exploit this property, e.g., for efficient nearest 
neighbour search in graph databases.

\section*{Acknowledgements}
This work was supported by the German Research Foundation (DFG) within the 
Collaborative Research Center SFB 876 ``Providing Information by 
Resource-Constrained Data Analysis'', project A6 ``Resource-efficient Graph 
Mining''.

\bibliographystyle{IEEEtranN}
\bibliography{lit}

\end{document}